\documentclass[runningheads]{llncs}

\usepackage{eccv}

\usepackage{eccvabbrv}

\usepackage{graphicx}
\usepackage{booktabs}

\usepackage[accsupp]{axessibility}  

\usepackage{hyperref}

\usepackage{orcidlink}
\usepackage{dsfont}

\begin{document}

\title{Sparsity-Inducing Divergence Losses for Biometric Verification} 


\author{Dimitrios Koutsianos\inst{1,2} \and
Ladislav Mošner\inst{3,4} \and
Yannis Panagakis\inst{2,5} \and
Themos Stafylakis\inst{1,2,4}}

\authorrunning{D.~Koutsianos et al.}


\institute{Athens University of Economics and Business, Greece \and
Archimedes/Athena Research Center, Greece \and
Speech@FIT, Brno University of Technology, Czechia \and
Omilia, Greece \and
National and Kapodistrian University of Athens, Greece\\
\email{\{dkoutsianos, tstafylakis\}@aueb.gr}}

\maketitle

\begin{abstract}
    Performance in face and speaker verification is largely driven by margin-penalty softmax losses such as CosFace and ArcFace. Recently introduced $\alpha$-divergence loss functions offer a compelling alternative, particularly due to their ability to induce sparse solutions (when $\alpha>1$). However, standard geometric margins are designed for the softmax function and do not naturally extend to this generalized probabilistic framework. 
    In this paper we propose Q-Margin, a novel $\alpha$-divergence loss that introduces a principled probabilistic margin. Unlike conventional methods that apply geometric penalties to the logits (unnormalized log-likelihoods), Q-Margin encodes the margin penalty directly into the reference measure (prior probabilities). This formulation naturally encourages discriminative embeddings while preserving the beneficial sparsity properties of the $\alpha$-divergence. We demonstrate that Q-Margin achieves competitive or superior performance on the challenging IJB-B and IJB-C face verification benchmarks and similarly strong results in speaker verification on VoxCeleb. Crucially, against ArcFace and CosFace baselines trained under an identical recipe, Q-Margin consistently improves at low False Acceptance Rates (FARs), a capability critical for practical high-security applications. Finally, the extreme sparsity of the Q-Margin posteriors enables exact and memory-efficient training, offering a scalable solution for datasets with millions of identities. 
\end{abstract}    
\section{Introduction}
\label{sec:intro}

Margin-penalty softmax losses such as CosFace~\cite{cosface} and ArcFace~\cite{arcface} have become the de facto standard in modern face and speaker recognition. They achieve state-of-the-art accuracy by incorporating margin penalties into the logit of the ground truth identity, which explicitly maximizes inter-class separation. Such methods rely on the $\operatorname{softargmax}$ operator, which transforms logits into probabilities evaluated by the Cross-Entropy loss or equivalently, the Kullback–Leibler (KL) divergence with a uniform prior~\cite{alpha-loss}. To date, margin penalties have been applied exclusively through geometric modifications to logits (\eg, $\cos(\theta + m)$ in ArcFace, $\cos(\theta) - m$ in CosFace). While effective, these geometric margins are heuristic extensions of softmax that remain decoupled from its probabilistic formulation.

Recent work by Roulet et al.~\cite{alpha-loss} introduced a principled 
generalization of the standard Cross-Entropy loss through Fenchel–Young 
losses~\cite{fenchel-young-losses} derived from $f$-divergences~\cite{f-divergence}. This framework extends Cross-Entropy in two key ways: (1)~by replacing the KL divergence with a broader family of $f$-divergences, each defining a unique convex loss and corresponding 
probability mapping (the $\operatorname{softargmax}_f$), and (2)~by allowing non-uniform reference measures $\mathbf{q}$ to encode arbitrary prior probabilities. While Roulet et al.\ used non-uniform $\mathbf{q}$ to represent class frequency imbalance, no prior work has exploited this flexibility to encode margin penalties probabilistically.

Among the $f$-divergence family, the $\alpha$-divergence is 
particularly useful for biometric verification due to its tunable 
trade-off between smooth and sparse probability mappings. It defines a 
continuum of generalized entropy measures: recovering the Shannon 
negentropy (standard Cross-Entropy) as $\alpha \rightarrow 1$, the Gini 
negentropy (SparseMax~\cite{sparsemax}) at $\alpha = 2$, and exhibiting 
intermediate behaviors for values such as $\alpha = 1.5$. For $\alpha > 1$, the $\alpha$-$\operatorname{softargmax}$ suppresses small probabilities, producing sparse posteriors that emphasize confident predictions. In practice, this encourages more discriminative embeddings and sharper decision boundaries in face and speaker recognition.

In this paper, we propose \textbf{Q-Margin}, a method that introduces a \emph{probabilistic margin} by encoding the penalty directly in the reference measure $\mathbf{q}$. Instead of modifying the logits geometrically, we modify the prior for the ground truth class. By down-weighting the reference measure $q_y$ of the target class, we create a stricter learning objective that naturally enforces a margin. This formulation preserves the sparsity benefits of $\alpha$-divergence, allowing efficient computation of the normalizer of the posterior probabilities (which requires an iterative algorithm for $\alpha \neq 1$). In the limiting case $\alpha \to 1$, Q-Margin reduces to CosFace.

By reinterpreting margin learning through the lens of $\alpha$-divergence, we unify geometric and probabilistic perspectives within a single framework that is theoretically grounded. This bridges a fundamental gap between information-theoretic loss design and the empirical success of margin-penalty softmax.\\
\textbf{Our contributions are as follows:}
\begin{itemize}
    \item We propose $\alpha$-divergence losses with margin penalties 
    as an inherently-sparse alternative to standard logistic loss, for training highly discriminative identity embeddings.
    \item We derive Q-Margin, which introduces a principled probabilistic margin through non-uniform reference measures. We show that this formulation recovers the popular CosFace loss as $\alpha \rightarrow 1$, providing a theoretical link between probabilistic and geometric margins.
    \item We exploit the sparsity of the posteriors for $\alpha>1$ and propose a simple modification of the iterative bisection algorithm (used to estimate the normalizer of posterior probabilities), which substantially accelerates training in large-scale datasets.
    \item We comprehensively evaluate on WebFace42M, IJB-B, IJB-C, and VoxCeleb, showing that Q-Margin is competitive with state-of-the-art methods including AdaFace~\cite{adaface} and PartialFC~\cite{partialfc}, and is particularly strong at the low FARs that matter for high-security verification.
    \item We provide a detailed empirical analysis of hyperparameter sensitivity, specifically investigating the interplay between $\alpha$, scale ($s$) and margin ($m$) in different biometric modalities.
\end{itemize}
\section{Related Work}
\label{sec:related_work}

\subsection{Margin-Penalty Losses and Scalable Classifiers for Biometric Verification}

Early deep face recognition models utilized the standard Softmax loss \cite{VGGFace2}, but this approach was found to be insufficient for learning features with enough discriminative power for open-set recognition tasks. This limitation spurred the development of loss functions that explicitly engineer the feature space to increase inter-class variance while minimizing intra-class variance.

L-Softmax \cite{L-Softmax} was a pioneering effort that introduced an angular margin into the loss function, forcing decision boundaries to be more compact. SphereFace \cite{sphereface} further refined this by normalizing the final layer's weights and constraining features to a hypersphere with a multiplicative angular margin, though both methods proved challenging to train. The field subsequently shifted towards additive margins, which proved more stable and effective. CosFace \cite{cosface} proposed adding a margin directly in cosine space ($\cos(\theta) - m$), while ArcFace~\cite{arcface} added an angular margin to the angle itself ($\cos(\theta + m)$). ArcFace in particular has become a widely used baseline in both face and speaker verification~\cite{xiang2019_margin-matters}, with consistent adoption across VoxSRC challenges~\cite{huh2024_voxsrc-retrospective} and extensions such as large margin fine-tuning~\cite{Thienpondt21:lm-finetuning}.

A common assumption in these fixed-margin methods is that all identities and samples present the same degree of learning difficulty. Recent adaptive margin strategies relax this assumption. AdaFace~\cite{adaface} dynamically adjusts the margin per sample using the feature norm as a quality proxy. KappaFace~\cite{KappaFace} assigns class-level margins via von Mises-Fisher concentration estimates. ElasticFace~\cite{ElasticFace} replaces the fixed margin with a stochastic one drawn from a Gaussian distribution. While these methods improve flexibility, they remain coupled to the standard softmax Cross-Entropy and introduce additional components (quality estimators, distribution parameters) that are orthogonal to our contribution. Our approach instead generalizes the underlying divergence, and a probabilistic margin formulation like Q-Margin could in principle be combined with such adaptive schemes in future work.

On the scalability front, PartialFC~\cite{partialfc2, partialfc} addresses the computational bottleneck of the final fully-connected layer by sampling a subset of negative class centers each training step, reducing cost while maintaining accuracy. Our work offers a complementary perspective: the extreme posterior sparsity induced by $\alpha$-divergences ($\alpha > 1$) naturally concentrates computation on a small active set of classes, a similar efficiency gain that follows as a mathematical consequence of the loss rather than a sampling heuristic.

Despite their empirical success, all of the above methods apply margin penalties through \emph{geometric} modifications to the logits before passing them to the standard Cross-Entropy loss. This leaves unexplored the question of whether margins can be introduced \emph{within} a more general probabilistic loss framework, a gap our work aims to address.

\subsection{Generalizations of the Cross-Entropy Loss}

A separate line of research has focused on generalizing the Cross-Entropy loss from an information-theoretic standpoint. Roulet et al. \cite{alpha-loss} utilized Fenchel-Young losses \cite{fenchel-young-losses} to derive a broad family of convex losses from $f$-divergences, providing a unified perspective that recovers, among others, the SparseMax transformation \cite{sparsemax}, a sparse alternative to $\operatorname{softargmax}$ capable of producing distributions with exact zeros. Chan and Kittler \cite{angularsparsemax} combined SparseMax with ArcFace margins using MobileFaceNets~\cite{mobilefacenets}, achieving competitive results. Peters et al. further generalized SparseMax with EntMax~\cite{entmax}, creating a family of transformations with controllable sparsity parameterized by the $\alpha$-divergence. Other information-theoretic generalizations include R\'{e}nyi divergence objectives for variational inference \cite{renyi1961measures, li2016renyi} and $f$-divergence-based training criteria \cite{f-divergence-novello}, both demonstrating the flexibility of alternative divergences for learning.

Our work builds on the $\alpha$-divergence loss from this framework, but addresses a different question. Prior uses of these generalized losses have either kept the reference measure uniform or used it to model class frequency imbalance \cite{alpha-loss}. No existing work has exploited the reference measure to encode \emph{margin penalties}, the mechanism that has driven the empirical success of face and speaker verification systems. Q-Margin bridges these two research directions by introducing a principled probabilistic margin within the $\alpha$-divergence framework, recovering CosFace as the special case $\alpha \to 1$.
\section{Method}

\subsection{Notation}

Let $k$ be the number of identities in the training set. We denote $[k] = \{1,\dots,k\}$. We denote the probability simplex by $\Delta^{k-1} = \{ \textbf{p}\in\mathbb{R}^{k}_{+}: \langle\textbf{p},\textbf{1}\rangle = 1\}$. We also denote by $\boldsymbol{\theta} \in \mathbb{R}^k$ the output logits produced by a neural network, and $\textbf{y} \in \{e_1,\dots,e_k\}$ denotes the one-hot identity labels. 

Throughout this work, we adopted the convention proposed by Roulet et al.~\cite{alpha-loss}:\\
\noindent
\begin{minipage}{0.44\linewidth}
\begin{equation}
\operatorname{softmax}(\boldsymbol{\theta}) = \log\sum_{j'=1}^k \exp(\theta_{j'})
\end{equation}
\end{minipage}
\hfill
\begin{minipage}{0.52\linewidth}
\begin{equation}
[\operatorname{softargmax}(\boldsymbol{\theta})]_j = \frac{\exp(\theta_j)}{\sum_{j'=1}^k\exp(\theta_{j'})}
\end{equation}
\end{minipage}
This convention distinguishes the log-partition function, $\operatorname{softmax}(\boldsymbol{\theta})$, from its gradient, the standard probability-generating $\operatorname{softmax}$ function, $\operatorname{softargmax}(\boldsymbol{\theta})$, as shown in \cref{eq:softmax_grad}. The $\operatorname{softmax}$ is the log-sum-exp operator and $\operatorname{softargmax}$ is the standard $\operatorname{softmax}$ function.
\begin{equation}
    \nabla_{\boldsymbol{\theta}}\operatorname{softmax}(\boldsymbol{\theta}) = \operatorname{softargmax}(\boldsymbol{\theta})
    \label{eq:softmax_grad}
\end{equation}
This equation also holds for the corresponding $\operatorname{softmax}_f$ and $\operatorname{softargmax}_f$ functions discussed below~\cite{alpha-loss}.

\subsection{$\boldsymbol{f}$- and $\boldsymbol{\alpha}$-divergence Loss Framework}

Our methods build upon the generalized framework of Fenchel-Young losses~\cite{fenchel-young-losses}, specifically using the $f$-divergence loss, $l_f$, as proposed by Roulet et al.~\cite{alpha-loss}. 

This framework is derived from a fundamental optimization problem. The generalized operator, $\operatorname{softmax}_f(\boldsymbol{\theta})$ is defined as the solution to a regularized maximization problem:
\begin{equation}
    \operatorname{softmax}_f(\boldsymbol{\theta}) = \max_{\textbf{p}\in\Delta^{k-1}} \langle\textbf{p},\boldsymbol{\theta} \rangle - D_f(\textbf{p}:\textbf{q})
\end{equation}
Here, the $f$-divergence, $D_f(\textbf{p}:\textbf{q}) = \langle f(\textbf{p}/\textbf{q}), \textbf{q}\rangle$, where $f(\cdot)$ is a convex function, acts as a regularizer. This generalizes the standard $\operatorname{softmax}$ (log-sum-exp) which is recovered when the KL divergence, $D_{KL}(\textbf{p}:\textbf{q})$, is used. The loss function is then defined as:
\begin{equation}
    l_f(\boldsymbol{\theta},\textbf{y};\textbf{q}) = \operatorname{softmax}_f(\boldsymbol{\theta}) + D_f(\textbf{y}:\textbf{q}) - \langle\boldsymbol{\theta}, \textbf{y}\rangle
\end{equation}
where $\boldsymbol{\theta}\in\mathbb{R}^k$ are the logits, $\textbf{y}\in\Delta^{k-1}$ is the ground-truth (in our case a one-hot vector with $1$ at index $y$), $\textbf{q}\in\mathbb{R}^k_+$ is the reference measure (class priors). Note that the term $D_f(\textbf{y}:\textbf{q})$ is constant with respect to the logits $\boldsymbol{\theta}$ and therefore does not influence the gradients. 

The $\alpha$-divergence is a subclass of $f$-divergence, generated by the convex function:
\begin{equation}
    f(u) = \frac{(u^{\alpha} - 1) - \alpha(u-1)}{\alpha(\alpha - 1)}, \quad \alpha\neq 1
    \label{eq:f}
\end{equation}
This framework generalizes the standard Cross-Entropy loss (recovered as $\alpha\to1$ with $\textbf{q}=\textbf{1}$, yielding $f(u) = u\log u -(u-1)$) and allows for controllable sparsity for $\alpha>1$. We refer the reader to \cite{alpha-loss} for the full derivation.

\subsection{Margin-Penalty Losses}

The standard Cross-Entropy loss, while ensuring separability, proved insufficient to learn the highly discriminative features required for open-set identity recognition. Early approaches like L-Softmax \cite{L-Softmax} and SphereFace~\cite{sphereface} introduced angular margins but faced training issues regarding both numerical stability and convergence. Subsequent research shifted towards additive margins, which proved to be more stable and effective in practice. 

Two prominent examples dominated the field: CosFace~\cite{cosface} (additive cosine margin) and ArcFace~\cite{arcface} (additive angular margin). Both methods penalize the target logit to force the model to learn distinctly separated features, thereby increasing discriminative power. However, these successful approaches share a common mechanism: they directly modify the target logit geometrically before it is passed to the loss function. This geometric manipulation contrasts with Q-Margin, which achieves a similar discriminative effect probabilistically by modifying the reference measure $\textbf{q}$ within the $\alpha$-divergence framework.

\subsection{Q-Margin: Modifying the Reference Measure}

Building on the generalized $\alpha$-divergence loss framework, we propose \textbf{Q-Margin}, which introduces the margin penalty through the reference measure $\textbf{q}$. Instead of directly manipulating the scaled cosine similarities $\boldsymbol{\theta} = s\cdot \textbf{c}$ (where $\textbf{c}\in [-1,1]^k$ is the vector of cosine similarities between a sample's embedding and the $k$ class prototype vectors), we encode the margin penalty into $\textbf{q}$. For a sample with ground-truth label $y$, the reference measure is defined as:

\begin{equation}
\label{eq:prior}
    q_j = \exp{(-s\cdot m \, \delta_{y,j})}
\end{equation}
where $s$ is the scaling factor, $m$ is the margin hyperparameter and $\delta_{\cdot,\cdot}$ is the Kronecker delta function. By significantly down-weighting $q_y$, we create a stricter learning objective. Minimizing the $\alpha$-divergence loss $l_f(\boldsymbol{\theta},\textbf{y};\textbf{q})$ compels the model to produce a substantially higher target logit $\theta_y$ relative to non-target logits to compensate, implicitly creating the desired decision margin through probabilistic means rather than direct geometric logit modification. The final loss is
\begin{equation}
    L_{Q-M} = l_f(s \cdot \textbf{c}, \textbf{y}; \textbf{q}).
\end{equation}
It is easy to show that CosFace is a special case of Q-Margin as $\alpha \to 1$. The reference-measure margin in~\cref{eq:prior} is orthogonal to a geometric angular margin and the two can be combined, but in preliminary experiments on MS1MV3 and WebFace42M an added angular shift yielded no gain, so we retain the purely probabilistic formulation.

\subsubsection{Gradients and Posterior Sparsity} \label{sec:grads}

The gradients for the $\alpha$-divergence loss, $l_f$, are key to understanding the model's behavior. The gradient w.r.t. the $j$-th logit $\theta_j$ (assuming target identity $y$) is $\frac{\partial l_f}{\partial\theta_j}=[\operatorname{softargmax}_f(\boldsymbol{\theta})]_j - {\bf y}_j$. The $j$-th element of the $\alpha$-$\operatorname{softargmax}$, which we denote as the posterior probability $p_j$, is given by: 
\begin{equation}
    p_j = [\operatorname{softargmax}_f(\boldsymbol{\theta})]_j 
    = q_{j}[1+({\alpha-1})(\theta_j - \tau^*)]_{+}^{1/(\alpha-1)}
\label{eq:l_grad_theta}
\end{equation}
where $\tau^*$ is a scalar value, computed iteratively, that ensures the probabilities sum to one $\left(\sum_jp_j = 1\right)$. For $\alpha>1$, the $[\cdot]_+ = \max(0,\cdot)$ operator is active, meaning that if a logit $\theta_j$ falls below a certain threshold relative to $\tau^*$, the posterior probability $p_j$ becomes exactly 0, resulting in sparse posterior distributions.

It is instructive to note the relationship between this formulation and the standard $\operatorname{softargmax}$. As $\alpha\to1$, the $p_j$ term recovers the standard exponential function $\exp(x)\!=\!\lim_{n\to \infty}\left(1+\frac{x}{n}\right)^n$. By setting $x\!=\!\theta_j\!-\!\tau^*$ and $n\!=\!1/(\alpha-1)$, we see that $\lim_{\alpha \to 1}p_j = \exp{(\theta_j - \tau^* + \log q_j)}$. This recovers the form of the standard $\operatorname{softargmax}$ probability, making the $\alpha$-divergence loss a polynomial generalization of the standard Cross-Entropy loss, and the CosFace loss when $\textbf{q}$ is as in \cref{eq:prior}.

\subsubsection{Efficient Computation of $\boldsymbol{\alpha}$-$\boldsymbol{\operatorname{softargmax}}$}

The $\alpha$-$\operatorname{softargmax}$ (\cref{eq:l_grad_theta}) is computed by first finding the scalar $\tau^*$. This is the unique solution for $\tau$ in the root finding equation:
\begin{equation}
    \label{eq:tau}
    \sum_{j=1}^k q_jf_*'(\max\{\theta_j-\tau, f'(0)\}) = 1.
\end{equation}
Here, $f$ is the convex function generating the $\alpha$-divergence (\cref{eq:f}), $f'(u) = \frac{u^{\alpha-1}-1}{\alpha-1}$, $\alpha \neq 1$, the derivative of $f$, $f^*(v) = \frac{1}{\alpha}([1+(\alpha-1)v]^{\frac{\alpha}{\alpha-1}}_+ - 1)$, $\alpha\in(1,+\infty)$ is the convex conjugate of $f$, while $f_*'=(f^*)'$ is the derivative of the convex conjugate. The components $q_j$ are from the reference measure and $\theta_j$ are the input logits. 

Following Roulet et al.~\cite{alpha-loss}, this value $\tau^*$ is found efficiently with the iterative bisection algorithm. The search is bounded within the range $[\tau_{\min},\tau_{\max}]$ where for $t \in \operatorname{argmax}_{j \in [k]} \theta_j$ it is \ $\tau_{\min} = \theta_{t} - f'(1 / q_{t})$ and $\tau_{\max} = \theta_{t} - f'\left(1 / \sum\nolimits_{j} q_j\right)$.

A key advantage of the $\alpha$-divergence with $\alpha > 1$ is the extreme sparsity of the resulting posterior distribution. In our experiments, the vast majority of class probabilities are assigned \emph{exactly} zero mass (see~\Cref{sec:Computational-eff} for detailed statistics). We leverage this property to significantly reduce computational cost. Instead of computing the root-finding step over all $k$ identities (which can be in the millions for datasets like WebFace42M), we sort the logits (in half-precision) to retrieve the top-5\% of logits and compute $\tau^*$ solely on this active set. Since the remaining 95\% of logits would fall below the sparsity threshold regardless, this optimized procedure is exact, not approximate (\cref{prop:exact}), while substantially reducing the computational overhead of the iterative bisection.

\begin{proposition}[Exactness of top-$K$ truncation]
\label{prop:exact}
Fix $\alpha>1$ and let $\tau^\ast$ solve \cref{eq:tau} over all $k$ classes, with active support
$S=\{j: p_j>0\}=\{j:\theta_j>\tau^\ast-\tfrac{1}{\alpha-1}\}$. If $S\subseteq\mathcal{T}_K$, where
$\mathcal{T}_K$ indexes the $K$ largest logits, then solving \cref{eq:tau} over $\mathcal{T}_K$ alone
yields the same $\tau^\ast$ and the same posterior $\mathbf{p}$ as the full computation.
\end{proposition}

\noindent\emph{Proof sketch.} From \cref{eq:l_grad_theta}, $p_j>0 \iff \theta_j>\tau^\ast-\tfrac{1}{\alpha-1}$,
so every class outside $S$ contributes $p_j=0$ and drops out of the sum in \cref{eq:tau}. If
$S\subseteq\mathcal{T}_K$, the discarded logits are all inactive, hence restricting the root-finding to
$\mathcal{T}_K$ leaves $\tau^\ast$ and $\mathbf{p}$ unchanged. $\hfill\square$

We fix $K=5$\% throughout (not tuned per run). This is deliberately conservative: the active support never exceeded $\sim$0.4\% in any of our experiments, so the wide margin guarantees that $S \subseteq T_K$ holds comfortably rather than relying on a tight threshold. As an additional safeguard we verify $S \subseteq T_K$ at runtime by checking that the smallest retained logit is itself inactive. The bound was never violated in any of our experiments.
\section{Face Recognition Experiments}

\subsection{Training and Testing Data}

Our models were trained on WebFace42M~\cite{webface260m}, a cleaned subset of WebFace260M containing 42M images of 2M identities. Following standard face-recognition protocols, all images were aligned and resized to $112 \times 112$ pixels.
During training, several standard verification benchmarks served as development sets: LFW~\cite{lfw}, CFP-FP~\cite{cfp-fp}, AgeDB-30~\cite{agedb}, CALFW~\cite{ca-lfw} and CPLFW~\cite{cp-lfw}, for which we report True Acceptance Rate (TAR) at a False Acceptance Rate (FAR) of $10^{-3}$.

The final evaluation was performed on the more challenging IJB-B \cite{ijb-b} and IJB-C \cite{ijb-c} benchmarks. These datasets are designed to test recognition performance in unconstrained scenarios. On these two test sets, we report TAR at FAR levels of $10^{-4}$ and $10^{-5}$ as our primary performance metrics, which represent standard operating points for evaluating these benchmarks.

Due to the significant computational cost of training on WebFace42M, most configurations were trained a single time; we additionally repeated our reference configuration ($\alpha=1.25$, $s=32$, $m=0.2$) across three random seeds (\Cref{tab:multiseed}) to verify stability. We further mitigate the single-run limitation by reporting a broad hyperparameter sweep in \Cref{tab:results-face} and by validating trends across two distinct biometric domains.

\subsection{Implementation Details}
\setcounter{footnote}{0}
The experimental framework was built upon the InsightFace toolkit\footnote{\url{https://github.com/deepinsight/insightface} (accessed June 1, 2025)}, specifically its ArcFace implementation, with minimal modifications to integrate the custom $\alpha$-divergence loss. Experiments were conducted on two equivalent setups: AWS (8$\times$NVIDIA A100 40GB GPUs) and LUMI (4$\times$AMD MI250X GPUs, each exposing two GCDs for 8 logical devices).

We evaluated a standard backbone architecture from the ResNet family \cite{resnet}, the ResNet-100 (R-100) backbone. All models were trained for 20 epochs using Stochastic Gradient Descent (SGD) as an optimizer. We employed a step-wise learning rate decay schedule, starting at 0.1 for the first 7 epochs, decreasing to 0.01 for the next 6 epochs, and concluding at 0.001 for the final 7 epochs. A momentum of 0.9 and a weight decay of $5\times10^{-4}$ were used. The batch size was set to 128 per device (GCD), resulting in an effective total batch size of 1024.

The baseline losses (ArcFace, CosFace) use their canonical published settings for each modality.
For Q-Margin we grid-searched $\alpha$, $s$, and $m$; the grids differ by modality because the optimal
scale is coupled to $\alpha$ (discussed below), so face configurations explore larger $s$
than speaker ones. We report every configuration evaluated (Table~\ref{tab:results-face} for face; Table~\ref{tab:results-r34-vox2dev} for speaker), perform no held-out model selection, and evaluate the final-epoch checkpoint once on IJB-B/C. Baselines received comparable search effort over their standard $(s,m)$ ranges.

\subsection{Results}
\label{sec:results}

We present a comprehensive evaluation of Q-Margin against the industry-standard margin-penalty softmax losses ArcFace~\cite{arcface} and CosFace~\cite{cosface}, alongside recent state-of-the-art methods including AdaFace~\cite{adaface}, PartialFC~\cite{partialfc2}, UniFace~\cite{uniface}, and TopoFR~\cite{NEURIPS2024_419b6c97}.

As shown in~\Cref{tab:sota_comparison}, Q-Margin is the strongest method in the controlled (re-implemented) block, which is the only setting where loss-function differences are isolated from training data and backbone. Our best configuration ($\alpha=1.25$, $s=35$, $m=0.2$) reaches 93.67\% on IJB-B and 96.35\% on IJB-C at FAR=$10^{-5}$, improving over the ArcFace and CosFace baselines trained identically.

Against the reported numbers, Q-Margin is competitive rather than uniformly ahead. It essentially ties PartialFC on IJB-C (97.78\% vs.\ 97.82\% at $10^{-4}$; 96.35\% vs.\ 96.46\% at $10^{-5}$), while methods such as UniFace and TopoFR report higher IJB-C numbers using larger backbones (R-200) and additional architectural components. These gains are therefore not attributable to the loss alone, whereas Q-Margin obtains its results purely through a change in the objective, with no auxiliary quality network, negative sampling, or deeper backbone.

Notably, Q-Margin's near-parity with PartialFC is achieved while activating only $\sim$0.4\% of classes per example, suggesting the sparsity mechanism may play a role analogous to hard-negative mining.

\begin{table*}[h!]
    \centering
    \caption{\textbf{Comparison with state-of-the-art methods on IJB-B and IJB-C.} TAR (\%) at FAR=$10^{-4}$ and $10^{-5}$. \emph{Reported} numbers are taken from the original papers and differ in training corpus and/or backbone, so they serve as literature context rather than controlled comparisons. \emph{Re-implemented} methods share an identical recipe (ResNet-100, WebFace42M), isolating loss-function differences; best in bold. EntMax and SparseMax also use an ArcFace-style angular margin. $^{\dagger}$Backbone depth not specified in the original paper.}
    \label{tab:sota_comparison}
    \resizebox{\textwidth}{!}{
    \begin{tabular}{l @{\hskip 1em} c @{\hskip 1em} c @{\hskip 1em} c @{\hskip 1em} c@{\hskip 0.5em}c @{\hskip 1em}c@{\hskip 0.5em}c}
        \toprule
         & & & & \multicolumn{2}{c}{\textbf{IJB-B}} & \multicolumn{2}{c}{\textbf{IJB-C}} \\
         \cmidrule(lr){5-6} \cmidrule(lr){7-8}
         Method & Train Data & Backbone & Venue & $10^{-4}$ & $10^{-5}$ & $10^{-4}$ & $10^{-5}$ \\
        \midrule
        \multicolumn{8}{l}{\textit{Reported}} \\
        AdaFace~\cite{adaface}            & WF4M  & R-100 & CVPR22  & 96.03 & -     & 97.39 & -     \\
        PartialFC (0.3)~\cite{partialfc2} & WF42M & R-100 & CVPR22 & 96.47 & - & 97.82 & 96.46 \\
        CAFace + AdaFace~\cite{caface}    & WF4M  & R-100 & NeurIPS22 & 95.53 & 92.29 & 97.30 & 95.96 \\
        F$^2$C~\cite{ffc}                 & WF42M & R-100 & CVPR22  & -     & -     & 97.31 & -     \\
        QAFace~\cite{qaface}              & MS1MV2 & ResNet$^{\dagger}$ & WACV23  & 95.67 & -     & 97.20 & -     \\
        CurricularFace~\cite{huang2020curricularface} & MS1MV2  & R-100 & CVPR20  & 94.80  & -  & 96.10  & -  \\
        MagFace~\cite{Meng_2021_CVPR}     & MS1MV2 & R-100 & CVPR21  & 94.51  & 90.36  & 95.97  & 94.08  \\
        ElasticFace~\cite{ElasticFace} & MS1MV2 & R-100 & CVPR22 & 95.43  & -  & 96.65  & -  \\
        UniFace~\cite{uniface}            & WF42M & R-200 & ICCV23  & -  & -  & 97.91  & 96.68  \\
        TopoFR~\cite{NEURIPS2024_419b6c97}              & WF42M & R-200 & NeurIPS24 & - & -  & 98.01  & 97.10  \\
        \midrule
        \multicolumn{8}{l}{\textit{Re-Implemented}} \\
        ArcFace ($s\!=\!64$, $m\!=\!0.5$)              & WF42M & R-100 & - & 96.09 & 92.26 & 97.54 & 95.70 \\
        CosFace ($s\!=\!64$, $m\!=\!0.5$)              & WF42M & R-100 & - & 96.05 & 93.15 & 97.49 & 95.82 \\
        EntMax ($\alpha\!=\!1.25$, $s\!=\!64$, $m\!=\!0.5$)& WF42M & R-100 & - & 96.13 & 93.06 & 97.55 & 96.11 \\
        SparseMax ($\alpha\!=\!2$, $s\!=\!1.9$, $m\!=\!0.2$)& WF42M & R-100 & - & 95.50 & 91.84 & 97.12 & 95.28 \\
        \textbf{Q-Margin} ($\alpha\!=\!1.25$, $s\!=\!35$, $m\!=\!0.2$) (\textbf{Ours}) & WF42M & R-100 & - & \textbf{96.44} & \textbf{93.67} & \textbf{97.78} & \textbf{96.35} \\
        \bottomrule
    \end{tabular}}
\end{table*}

The Detection Error Tradeoff (DET) curves in \Cref{fig:det-curve} provide a visual corroboration of these results. Our best Q-Margin configuration (blue line) is shown to consistently outperform the ArcFace (pink) and CosFace (red) baselines. 

Beyond quantitative metrics, \Cref{fig:clean_grid} provides qualitative evidence of Q-Margin's improved robustness. The examples show our method succeeding where ArcFace fails, correctly verifying genuine pairs despite large intra-class variations (\eg extreme pose, aging) and correctly rejecting impostor pairs. This suggests Q-Margin learns a more robust and discriminative embedding space.

\paragraph{Hyperparameter Sensitivity}

Standard margin-penalty softmax losses, like ArcFace, typically employ a large margin penalty ($m=0.5$) combined with a high scaling factor ($s=64$). Our experiments indicate that this configuration is suboptimal for Q-Margin. As shown in \Cref{tab:results-face}, fixing $\alpha=1.25$ and $s=32$ while increasing $m$ from $0.2$ to $0.5$ results in a performance drop on both IJB-B and IJB-C. This suggests that the $\alpha$-divergence loss formulation is sensitive to aggressive probabilistic-margin penalties. We find that a moderate margin ($m=0.2$) yields the most robust results. To check that these gains reflect a stronger objective rather than seed variance, we repeated the $\alpha=1.25$, $s=32$, $m=0.2$ configuration across three seeds (\Cref{tab:multiseed}); the maximum standard deviation across IJB-B/C metrics is $0.07$ and the ordering over the ArcFace and CosFace baselines holds in every seed.

Furthermore, we observe a distinct coupling between the divergence parameter $\alpha$ and the scaling
factor $s$. Fixing $m=0.2$, the optimal scale depends strongly on $\alpha$: for $\alpha=1.25$ performance
improves steadily as $s$ increases, whereas for $\alpha=1.5$ it stagnates or degrades at larger scales
(\cref{tab:results-face}). This coupling admits a simple mechanistic explanation: increasing $\alpha$ makes the
$\operatorname{softargmax}_f$ mapping peakier at a fixed scale, and raising $s$ has the same sharpening
effect, so the two trade off along a single axis and the optimal scale \emph{decreases} as $\alpha$ grows
rather than varying arbitrarily. Accordingly, $\alpha$ and $s$ cannot be tuned independently, as increasing
both at once destabilizes the optimization. Therefore we adopt $\alpha=1.25$, $s=35$, $m=0.2$ as the effective
operating point on WebFace42M, achieving the best results while maintaining training stability.

\begin{table*}[h!]
    \centering
    \caption{Results for the WebFace42M training dataset. For all configurations we report the single run TAR.}
    \label{tab:results-face}
    \resizebox{\textwidth}{!}{%
        \begin{tabular}{l @{\hskip 0.5em} ccc @{\hskip 1em} ccccc @{\hskip 1em} c@{\hskip 0.5em}c @{\hskip 1em} c@{\hskip 0.5em}c}
            \toprule
            & \multicolumn{3}{c}{Parameters} & \multicolumn{5}{c}{Benchmarks} & \multicolumn{2}{c}{IJB-B} & \multicolumn{2}{c}{IJB-C}\\
             \cmidrule(lr){2-4} \cmidrule(lr){5-9} \cmidrule(lr){10-11} \cmidrule(lr){12-13} 
            Loss & $\alpha$ & $s$ & $m$ & LFW & CFP-FP & AgeDB-30 & CALFW & CPLFW & $10^{-4}$ & $10^{-5}$ & $10^{-4}$ & $10^{-5}$\\
            \midrule
ArcFace &  - & 32 & 0.5 & \bf{99.88} & 99.36 & 98.35 & 96.15 & \bf{94.90} & 95.75 & 91.69 & 97.42 & 95.36 \\
ArcFace &  - & 64 & 0.5 & 99.85 & 99.36 & 98.43 & 96.18 & 94.73 & 96.09 & 92.26 & 97.54 & 95.70 \\
CosFace &  - & 32 & 0.5 & 99.87 & 99.24 & 98.22 & 96.03 & 94.77 & 95.81 & 92.50 & 97.41 & 95.52 \\
CosFace &  - & 64 & 0.5 & 99.82 & 99.20 & 98.28 & 96.08 & 94.52 & 96.05 & 93.15 & 97.49 & 95.82 \\
Q-Margin &  1.25 & 10 & 0.2 & 99.87 & \bf{99.39} & 98.42 & \bf{96.25} & 94.87 & 96.08 & 92.95 & 97.62 & 95.91 \\
Q-Margin &  1.25 & 10 & 0.5 & 99.87 & 99.36 & 98.33 & 96.15 & 94.80 & 96.03 & 91.79 & 97.55 & 95.72 \\
Q-Margin &  1.25 & 20 & 0.2 & 99.85 & 99.34 & 98.45 & 96.22 & \bf{94.90} & 96.21 & 92.95 & 97.58 & 95.91 \\
Q-Margin &  1.25 & 20 & 0.5 & 99.83 & 99.36 & 98.30 & 96.17 & 94.80 & 96.02 & 92.45 & 97.57 & 95.92 \\
Q-Margin &  1.25 & 32 & 0.2 & 99.87 & 99.36 & \bf{98.57} & 96.15 & 94.83 & 96.30 & 93.38 & 97.76 & \textbf{96.35} \\
Q-Margin &  1.25 & 32 & 0.5 & 99.85 & 99.34 & 98.33 & 96.13 & 94.68 & 96.05 & 92.67 & 97.51 & 96.00 \\
Q-Margin & 1.25 & 35 & 0.2 & 99.83 & \textbf{99.39} & \textbf{98.57} & 96.08 & 94.78 & \textbf{96.44} & \textbf{93.67} & \textbf{97.78} & \textbf{96.35} \\
Q-Margin &  1.5 & 10 & 0.2 & 99.85 & 99.33 & 98.28 & 96.22 & 94.65 & 95.86 & 92.26 & 97.48 & 95.64 \\
Q-Margin &  1.5 & 10 & 0.5 & 99.87 & 99.29 & 98.15 & 96.08 & 94.42 & 95.56 & 91.65 & 97.07 & 95.04 \\
Q-Margin &  1.5 & 20 & 0.2 & 99.83 & 99.30 & 98.00 & 96.15 & 94.42 & 95.71 & 91.85 & 97.21 & 95.59 \\
Q-Margin &  1.5 & 20 & 0.5 & 99.83 & 99.29 & 98.27 & 96.17 & 94.58 & 95.86 & 92.94 & 97.38 & 95.92 \\
Q-Margin &  1.5 & 32 & 0.2 & 99.85 & 99.29 & 98.30 & 96.10 & 94.67 & 96.04 & 92.15 & 97.49 & 95.93 \\
Q-Margin & 1.5 & 32 & 0.5 & 99.85 & 99.37 & 98.35 & 96.12 & 94.70 & 95.99 & 92.27 & 97.43 & 95.76 \\
\bottomrule
        \end{tabular}%
    }
\end{table*}

As a practical guideline we recommend starting with $\alpha=1.25$ and a moderate margin ($m=0.2$), then tuning $s$ upward, since performance is more sensitive to margin than scale across almost all configurations we evaluated. Higher $\alpha$ (e.g.\ $\alpha=1.5$) gave diminishing returns at this scale but proved beneficial in the smaller speaker verification setting~(\Cref{sec:speaker}).

\subsection{Computational Efficiency}
\label{sec:Computational-eff}

\begin{table*}[b!]
    \centering
    \caption{\textbf{Multi-seed results for Q-Margin} ($\alpha\!=\!1.25$, $s\!=\!32$, $m\!=\!0.2$). Mean\,$\pm$\,standard deviation over three runs with distinct random seeds. The maximum standard deviation across IJB-B/C metrics is $0.07$, confirming the reported gains are stable across seeds.}
    \label{tab:multiseed}
    \resizebox{\textwidth}{!}{%
        \begin{tabular}{ccccc @{\hskip 1em} c@{\hskip 0.5em}c @{\hskip 1em} c@{\hskip 0.5em}c}
            \toprule
            \multicolumn{5}{c}{Benchmarks} & \multicolumn{2}{c}{IJB-B} & \multicolumn{2}{c}{IJB-C}\\
            \cmidrule(lr){1-5} \cmidrule(lr){6-7} \cmidrule(lr){8-9} 
            LFW & CFP-FP & AgeDB-30 & CALFW & CPLFW & $10^{-4}$ & $10^{-5}$ & $10^{-4}$ & $10^{-5}$\\
            \midrule
$99.83\!\pm\!0.03$ & $99.35\!\pm\!0.02$ & $98.51\!\pm\!0.06$ & $96.13\!\pm\!0.02$ & $94.87\!\pm\!0.04$ & $96.38\!\pm\!0.07$ & $93.43\!\pm\!0.05$ & $97.77\!\pm\!0.01$ & $96.34\!\pm\!0.02$\\
        \bottomrule
        \end{tabular}%
    }
\end{table*}

A potential concern with the $\alpha$-divergence framework is the computational overhead of the $\alpha$-$\operatorname{softargmax}$. Unlike the standard Softmax used in ArcFace, which benefits from highly optimized closed-form GPU implementations, our formulation requires an iterative bisection algorithm to find the root $\tau^*$. On extreme-scale datasets like WebFace42M, which contains over two million distinct identity classes, evaluating this root-finding algorithm across all logits simultaneously can introduce a noticeable computational bottleneck.

\begin{figure*}[h!]
    \centering
    \includegraphics[width=\textwidth]{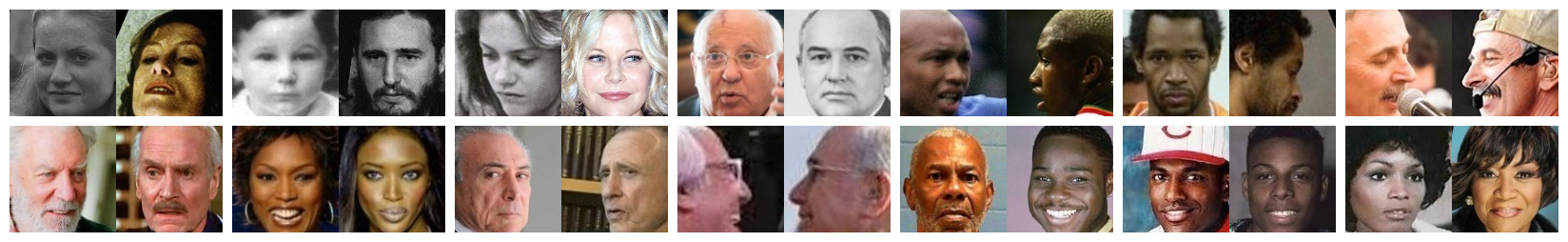}
    \caption{Challenging pairs from AgeDB-30, CALFW, and CPLFW where our proposed method (Q-Margin, $\alpha$=1.25, s=32, m=0.2) succeeds, while the ArcFace baseline fails. The top row shows genuine pairs correctly accepted by our model and falsely rejected by the baseline. The bottom row shows impostor pairs correctly rejected by our model and falsely accepted by the baseline.}
    \label{fig:clean_grid}
\end{figure*}

To quantify this, we benchmarked the training throughput (samples per second) on the WebFace42M dataset using the ResNet-100 backbone. As shown in \Cref{tab:throughput}, the naive implementation of Q-Margin evaluated over all 2M+ identities (Q-M 100\%) yields a throughput of 1466.45 samples/s, representing a roughly 27\% slowdown compared to the ArcFace baseline (2005.13 samples/s).

However, we mitigate this overhead by explicitly leveraging the extreme sparsity of the $\alpha$-divergence posterior ($\alpha > 1$). In practice, across all training examples on WebFace42M, the highest number of non-zero elements in $\textbf{p}$ was $7{,}976$ out of over 2M classes ($\sim0.4\%$), with an average of only 1,020 per example. We can therefore perform a partial sort to isolate the top-$K\%$ of logits and compute $\tau^*$ solely on this active set. Since this maximum is an order of magnitude below $K=5\%$, the active support is contained in the top-$K$ set, so the truncation is exact by~Proposition~\ref{prop:exact} rather than an approximation.


As demonstrated in~\Cref{tab:throughput}, restricting the bisection to the top 5\% of logits (Q-Margin 5\%) yields 1900.26 samples/s, a $\sim$5\% overhead versus ArcFace. We retain this $K\!=\!5$\% even though the observed active set never exceeds $\sim$0.4\%: the generous margin guarantees $S \subseteq T_K$ without relying on a tight bound that a rare example might exceed. Pushing to the top 1\% (Q-Margin 1\%) recovers throughput almost entirely (1978.07 samples/s, $\sim$1.3\% difference), so even this conservative choice scales to datasets with millions of identities.

\begin{table*}[h!]
    \centering
    \caption{Average throughput on WebFace42M ($>2$M identities) with ResNet-100. Q-Margin (1\%) means that we sort and keep only the top 1\% of the logits.}
    \begin{tabular}{c @{\hskip 1em} c @{\hskip 1em} c @{\hskip 1em} c}
    \toprule
        ArcFace & Q-Margin (1\%) & Q-Margin (5\%) & Q-Margin (100\%) \\
        \midrule
       2005.13  & 1978.07 & 1900.26  & 1466.45\\
       \bottomrule
    \end{tabular}
    \label{tab:throughput}
\end{table*}

\section{Speaker Verification Experiments}
\label{sec:speaker}
Beyond face verification, margin-penalty softmax losses are equally prevalent in speaker verification~\cite{xiang2019_margin-matters}, where speaker embeddings are increasingly used alongside visual features in audio-visual tasks such as speaker diarization~\cite{wuerkaixi2022dyvise}. We evaluate Q-Margin in this domain to test whether the benefits of probabilistic margins transfer across biometric modalities, particularly in a regime with far fewer training identities (5,994 vs.\ 2M).

\subsection{Training and Testing Data}

For speaker verification, we adopted the standard protocol established by the VoxCeleb benchmarks. The training corpus is the development set of VoxCeleb2~\cite{Voxceleb2}, comprising about 1M utterances from 5,994 speakers.

For evaluation, we adopted the challenging trial list termed VoxCeleb1-H (hard). It contains 550,894 trials featuring 1,190 speakers. The difficulty of the evaluation set arises from the constraint that trials consist of utterances from speakers of the same gender and nationality. While other trial lists (VoxCeleb1-O and VoxCeleb1-E) exist, we focus on this \emph{hard} set to maintain alignment with the challenging scenarios evaluated in the face recognition experiments. While experiments on the other trial lists yielded similar conclusions, we omit those results for brevity. Our analysis focuses on TAR(\%) at FAR of $10^{-3}$ and $10^{-4}$.

\begin{figure}[tb]
  \centering
  \begin{subfigure}{0.46\linewidth}
    \includegraphics[width=\linewidth]{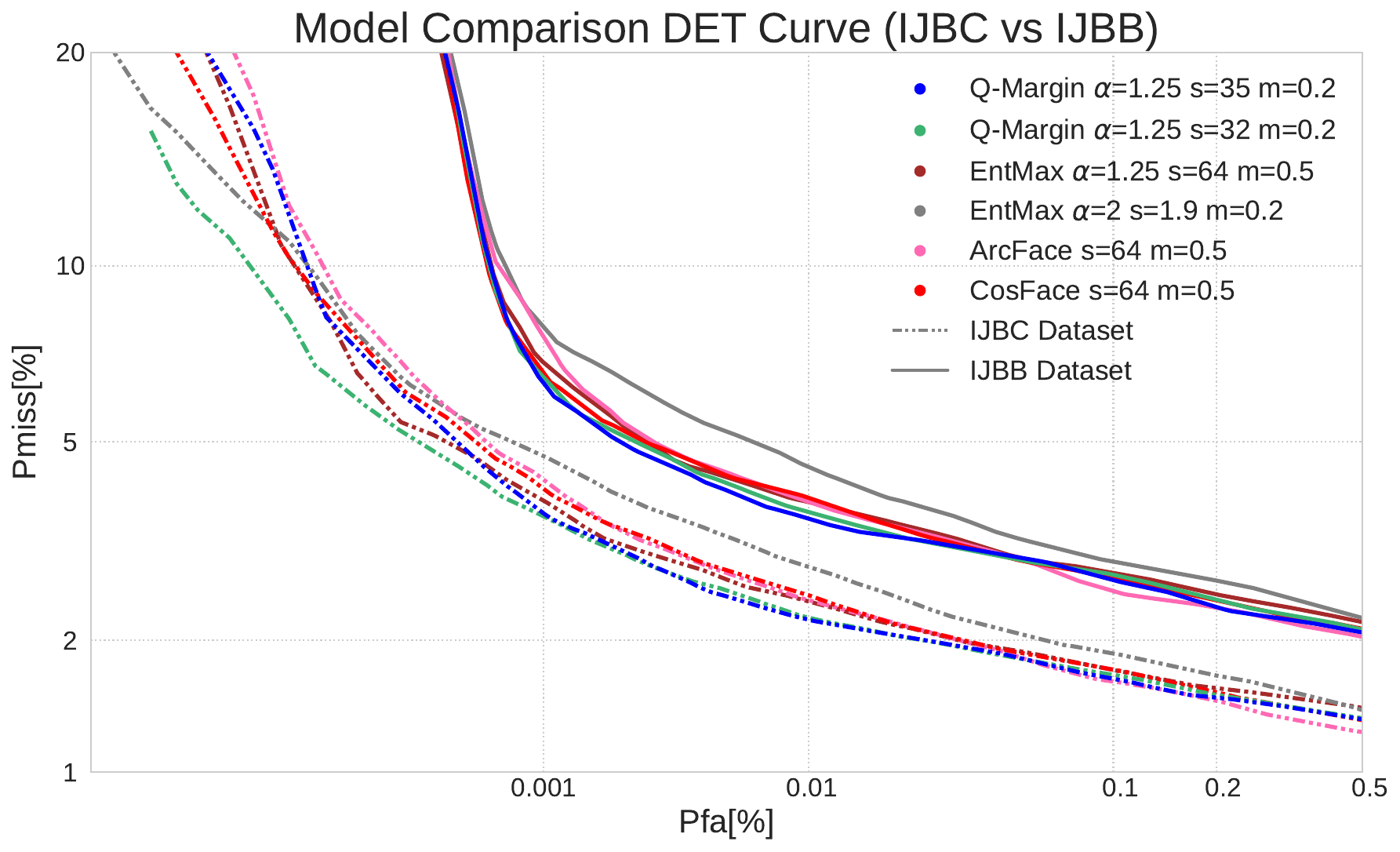}
    \caption{DET Curves for face recognition models.}
    \label{fig:det-curve}
  \end{subfigure}
  \hfill
  \begin{subfigure}{0.5\linewidth}
    \includegraphics[width=\linewidth]{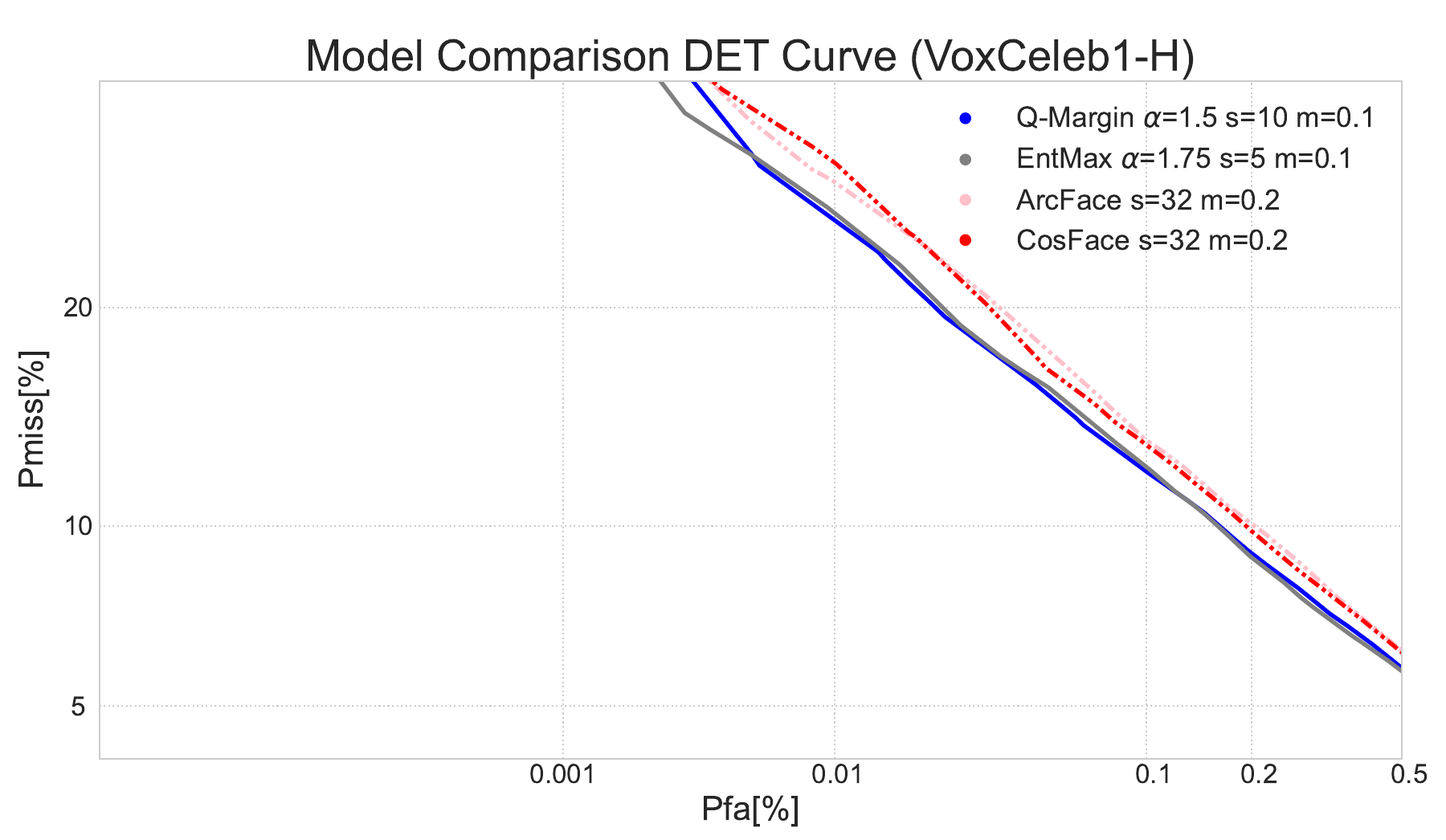}
    \caption{DET Curves for speaker recognition models.}
    \label{fig:det-curve-speaker}
  \end{subfigure}
  \caption{False acceptance vs missed detection probabilities for various losses on IJB-B and IJB-C (left) and VoxCeleb1-H (right), focused on low FARs.}
  \label{fig:short}
\end{figure}

\subsection{Implementation Details}

Our implementation is based on the WeSpeaker toolkit\footnote{\url{https://github.com/wenet-e2e/wespeaker}, (accessed September 30, 2025)}~\cite{wang2024advancing}. Training was conducted on 4 NVIDIA RTX A4000 (16GB) GPUs. A ResNet-34 (R-34) backbone was employed as the feature extraction network. The model was trained for 150 epochs using SGD with a momentum of 0.9 and a weight decay of $10^{-4}$. The learning rate followed a two-phase schedule: a linear warm-up during the first six epochs, after which it decreased exponentially from 0.1 to $5 \cdot 10^{-5}$ over the remaining training period. Following common practice in speaker verification, the margin parameter $m$ was annealed, starting from 0.0 and increasing exponentially between epochs 20 and 40 to its target value; this value was subsequently held constant. A batch size of 128 per GPU was used. To improve robustness, input data was augmented with additive noise and reverberation during training.

\subsection{Results}

Consistent with the face recognition evaluation, we compare our methods with the CosFace and ArcFace baselines. \Cref{tab:results-r34-vox2dev} reports TAR at different FAR values on VoxCeleb1-H, demonstrating the strong performance of our methods and their sensitivity to hyperparameters ($\alpha$, $s$, $m$). For a given $\alpha$, optimal $s$ and $m$ were found by empirical search over $s \in \{5, 10, 20, 32\}$ and $m \in \{0.1, 0.2, 0.3\}$.

\begin{table}[t]
    \centering
    \caption{Results with a ResNet-34 (R-34) model trained on VoxCeleb2 dev. We report TAR (\%) at FAR of $10^{-3}$ and $10^{-4}$. The experiments are performed using VoxCeleb1-H as test set.}
    \label{tab:results-r34-vox2dev} 
    
    \begin{tabular}{l @{\hskip 1em} ccc @{\hskip 1em} cc} 
        \toprule
         & \multicolumn{3}{c}{Parameters} &  \multicolumn{2}{c}{TAR@FAR} \\
        \cmidrule(lr){2-4} \cmidrule(lr){5-6} 
        Loss & $\alpha$ & $s$ & $m$ & $10^{-3}$ & $10^{-4}$ \\
        \midrule
        ArcFace & - & 32 & 0.2 & 86.62 & 72.08 \\
        CosFace & - & 32 & 0.2 & 86.68 & 70.61 \\
        \midrule
        EntMax & 2.0 & 5 & 0.1 & 84.71 & 74.37 \\
        EntMax & 1.75 & 5 & 0.1 & 87.69 & 74.08 \\
        Q-Margin & 1.75 & 5  & 0.1 & 87.02 & \textbf{74.60}\\
        Q-Margin& 1.5  & 10 & 0.1 & \textbf{87.91} & 72.96 \\
        \bottomrule
    \end{tabular}
\end{table}

In contrast to the face recognition results (which favored $\alpha=1.25$), we observed that higher values
of $\alpha$ ($1.5$ and $1.75$) lead to better-performing speaker verification systems. We hypothesize that
this is related in part to the significantly smaller number of training identities. However, class count
is not the whole story: although the speaker corpus has far fewer identities, modality-specific factors such as embedding-space geometry and intra-class variability also shape the ideal
sparsity level. Such modality-dependent tuning is not unique to Q-Margin since standard margin-penalty losses
exhibit the same gap, with ArcFace's optimum at $s=64, m=0.5$ for face versus $s=32, m=0.2$ for speaker.
Conversely, these results reinforce observations made in \Cref{tab:results-face}, i.e., that the most
consistent improvements are achieved with moderate values of scale ($s\in\{5,10\}$) and margin ($m=0.1$).
Specifically, the same $\alpha$--$s$ interdependence identified for face holds here: higher $\alpha$
requires smaller $s$, which is consistent with the peakiness mechanism of \Cref{sec:results} (\cref{eq:l_grad_theta}).

Both Q-Margin and EntMax yield strong results at the two operating points of interest. We observe a significant improvement from Q-Margin ($\alpha=1.5, s=10, m=0.1$) over both baselines, increasing the TAR to 87.91\% at $10^{-3}$ FAR. At the stricter $10^{-4}$ FAR, the Q-Margin ($\alpha=1.75$, $s=5$, $m=0.1$) formulation achieves the highest overall score (74.60\%).

To visually illustrate these outcomes, we present DET curves in the low-FAR region for selected systems in \Cref{fig:det-curve-speaker}. While the curves for CosFace and ArcFace nearly overlap, the curves for models using our proposed losses are visibly separated, indicating superior performance and robustness.

\section{Conclusions}
Margin-penalty softmax loss functions have driven the success of embedding-based networks in face and speaker verification, while recent advances in generalized Cross-Entropy ($\alpha$-divergence losses) have enabled principled control over posterior sparsity. However, combining these two paradigms is not straightforward: standard geometric margins, designed for the softmax function, do not naturally extend to sparsity-inducing $\alpha$-divergences.

To address this, we propose Q-Margin, a formulation that introduces a probabilistic margin by encoding penalties directly into the non-uniform reference measure $\mathbf{q}$. This approach bridges information-theoretic loss design with margin-based learning, providing a principled alternative to geometric logit manipulation that recovers CosFace as the special case $\alpha \to 1$.

Extensive evaluations on large-scale benchmarks (IJB-B/C, VoxCeleb1-H) demonstrate highly competitive performance, particularly at strict low-FAR operating points critical for high-security applications.  Furthermore, we showed that the inherent sparsity of Q-Margin can be leveraged to optimize the iterative bisection algorithm, yielding exact and highly efficient training throughput even on datasets with millions of identities. Our analysis also highlighted the interplay between $\alpha$, scale ($s$) and margin ($m$), offering practical guidance for tuning these generalized loss functions. We note that most of our WebFace42M experiments were run a single time due to computational cost, so the reported gains, while consistent across benchmarks and modalities, lack confidence intervals. Additionally, Q-Margin introduces one extra hyperparameter ($\alpha$) compared to standard margin-penalty softmax losses, and we observed that $\alpha$ and $s$ interact in ways that require joint tuning. 

\paragraph{Future work.} Natural extensions include dynamically updating the reference measure based on sample difficulty, connecting Q-Margin with adaptive strategies such as AdaFace, and evaluating on additional corpora for fully controlled comparisons.

\paragraph{Ethical considerations.}
Face and speaker verification rely on biometric data that is personal, hard to revoke once compromised, and prone to unequal error rates across demographic groups, and such models can be repurposed for surveillance without consent. Q-Margin's low-FAR, high-security focus can sharpen these concerns, so
deployments should use consensually collected data, audit demographic fairness at the relevant operating points, and follow applicable privacy regulation.

\section*{Acknowledgements}
This work has been partially supported by project MIS 5154714 of the National Recovery and Resilience Plan Greece 2.0 funded
by the European Union under the NextGenerationEU Program.\\
AWS resources were provided by the National Infrastructures for Research and Technology GRNET and funded by the EU Recovery and Resiliency Facility.\\
The work was supported by Ministry of Education, Youth and Sports of the Czech Republic (MoE) through the OP JAK project "Linguistics, Artificial Intelligence and Language and Speech Technologies: from Research to Applications" (ID:CZ.02.01.01/00/23\_020/0008518).

%
%
\bibliographystyle{splncs04}
\bibliography{main}
\end{document}